\documentclass[letterpaper]{article}
\usepackage{aaai18}
\usepackage{times}
\usepackage{helvet}
\usepackage{courier}
\usepackage{url}
\usepackage{graphicx}
\usepackage{balance}
\usepackage{amsmath,amsthm,amssymb,amsfonts}     
\usepackage{xspace}
\usepackage{color}      
\usepackage{mathtools}
\usepackage[position=b]{subcaption}
\usepackage{nicefrac}       
\usepackage{microtype}      
\usepackage{booktabs}       

\title{
	From Monte Carlo to Las Vegas:\\ 
	Improving Restricted Boltzmann Machine Training Through Stopping Sets
}
\author{ Pedro H. P. Savarese \\
   Toyota Technological Institute at Chicago\\
   Chicago, IL, 60637 \\
   \texttt{savarese@ttic.edu} \\
   \And
   Mayank Kakodkar\\
   Department of Computer Science\\
   Purdue University\\
   West Lafayette, IN, 47903\\
   \texttt{mkakodkar@purdue.edu}
   \And
   Bruno Ribeiro\\
   Department of Computer Science\\
   Purdue University\\
   West Lafayette, IN, 47903\\
   \texttt{ribeiro@cs.purdue.edu}
}

\newcommand{\EE}{\mathbb{E}}

\newcommand{\bx}{{\bf x}}
\newcommand{\bX}{{\bf X}}
\newcommand{\bW}{{\bf W}}
\newtheorem{theorem}{Theorem}

\newcommand{\one}{{\bf 1}}
\newtheorem{lemma}{Lemma}
\newtheorem*{lemma*}{Lemma}
\newcommand{\cL}{\mathcal{L}}

\newtheorem{definition}{Definition}
\newcommand{\bv}{{\bf v}}
\newcommand{\bh}{{\bf h}}
\newcommand{\cS}{\mathcal{S}}
\newcommand{\cC}{\mathcal{C}}
\newcommand{\cG}{\mathcal{G}}
\newcommand{\by}{{\bf y}}
\newcommand{\bS}{\textsf{\textsl{S}}}

\DeclareMathOperator*{\argmax}{\arg\!\max}
\newtheorem{corollary}{Corollary}

\newcommand{\ba}{{\bf a}}
\newcommand{\bb}{{\bf b}}
\newcommand{\cH}{\mathcal{H}}
\newcommand{\LVname}{{\em Las Vegas Slope}\xspace}
\newcommand{\LV}{LVS\xspace}
\newcommand{\tT}{T}
\newcommand{\cO}{O}

\begin{document}

\maketitle

\begin{abstract}
We propose a Las Vegas transformation of Markov Chain Monte Carlo (MCMC) estimators of Restricted Boltzmann Machines (RBMs). We  denote our approach {\em Markov Chain Las Vegas} (MCLV). MCLV gives statistical guarantees in exchange for random running times. 
MCLV uses a stopping set built from the training data and has maximum number of Markov chain steps $K$ (referred as MCLV-$K$). 
We present a MCLV-$K$ gradient estimator (LVS-$K$) for RBMs and explore the correspondence and differences between LVS-$K$ and Contrastive Divergence (CD-$K$), with LVS-$K$ significantly outperforming CD-$K$ training RBMs over the MNIST dataset, indicating MCLV to be a promising direction in learning generative models.
\end{abstract}

\section{Introduction}
Despite the significant recent advances in training discriminative neural network models,
training generative models has proven more elusive. 
As with most neural network training methods, algorithms for training Restricted Boltzmann Machines (RBMs)~\cite{Hinton2002,Hinton2012,smolensky1986information}, a class of energy-based generative neural network models, are unreasonably effective. Though, some argue, not yet effective enough for modern applications.
In this work we seek to better understand and 
improve the training of RBMs. 

RBM is a family of energy-based models with probability distribution over a state vector $\bx = (\bv,\bh)$ (assumed discrete w.l.o.g.),
$\bv \in \{0,1\}^{n_V}$ (e.g. an image) and binary latent variables $\bh \in \{0,1\}^{n_H}$,
\begin{equation} \label{q:energy}
p(\bx ; \bW) = \frac{1}{Z(\bW)} e^{-E(\bx; \bW)},
\end{equation}
where $Z(\bW) = \sum_{\bx} e^{-E(\bx; \bW)}$ is a partition function with finite mean, $\EE[Z(\bW)] < \infty$, and $E(\bx; \bW)$ is an energy function given by
\begin{equation*}
\begin{split}
E\Big(\bx = (\bv, \bh) &; \bW=(\bW',\bb,\ba)\Big) \\
				&= -\bv^\textsf{T} \bW' \bh - \bb^\textsf{T} \bv - \ba^\textsf{T} \bh .
\end{split}
\end{equation*}
RBMs proved highly successful in many tasks, such as data generation~\cite{Hinton2002,hinton2006fast,Hinton2012} and as a pre-training step for feedforward neural networks~\cite{Salakhutdinov2009}, among others (see Bengio and Delalleau~(\citeyear{Bengio2009}) and Erhan et al.~(\citeyear{erhan2010does})).

As computing $Z(\bW)$ directly is intractable for large state spaces, Markov Chain Monte Carlo (MCMC) methods are widely used to compute statistics of these models (including estimating the gradient $\partial p(\bx ; \bW)/ \partial \bW$). 
MCMC works by running a Markov chain (MC) $\Phi(\bW)$ with steady state $p(\bx; \bW)$ to equilibrium. 
Metropolis-Hastings and Gibbs sampling are two general such approaches.

However, in the real world, one is expected to run the MC $\Phi(\bW)$ for only $K$ steps, returning a state $\bx \sim \hat{\pi}(K)$, ``approximately sampled'' from the Markov chain's true steady state distribution $p(\bx ;\bW)$.
Starting from a random state, $K$ needs to be quite large for this method to work.

Contrastive Divergence (CD-$K$)~\cite{Hinton2002,Hinton2012}, improves this procedure by starting the MC from the visible states of the training data. 
Empirically, CD-$K$ works tremendously well to train RBMs with few hidden units ($n_H$ small) even for $K$ as low as $K = 1$~\cite{Carreira-Perpinan2005,Hinton2002,Hinton2012}.

For high-dimensional RBMs, CD-$K$ is less efficient and the reason is conjectured to be the longer mixing times~\cite{Sutskever2010}, although concrete evidence is anecdotal as mixing times are hard to assess in high dimensions. While not the main focus of our paper, armed with our techniques, we will further empirically explore possible reasons for this high-dimensional difficulty.

\paragraph{A Las Vegas transformation of RBM training.}
The main focus of this paper is to recast MCMC estimation of RBMs as a Markov chain algorithm with stopping sets obtained from the training data.
The size of the stopping set is a hyper-parameter that can be dynamically adapted during training based on computational trade-offs.

In standard RBM training using MCMC, the MC stops after a predefined number of $K$ steps. In our approach, the MCMC can also stop if it reaches one of the states in the stopping set.
Thus, MCMC running times are random (and are, in average, shorter than $K$).
This approach is closer to a Las Vegas algorithm than a Monte Carlo algorithm: we aim to get {\em perfect samples} of a quantity with an algorithm that has random running times.
We denote this approach Markov Chain Las Vegas with $K$ maximum steps (MCLV-$K$).

We show that, by dynamically adapting $K$, MCLV-$K$ can find unbiased estimates of the direction of the RBM gradient. 
Moreover, in contrast to standard MCMC, MCLV-$K$ has an extra piece of information: whether or not the stopping set has been reached.
We show that this knowledge provides novel ways to estimate gradients and partition functions of energy-based models.

And perhaps, one of the most interesting observations in this paper comes from the {\bf correspondence between CD-$\boldsymbol{K}$ and MCLV-$\boldsymbol{K}$}.
MCLV-1 is quite similar to CD-1 except for an added {\em $\cS$-stopped} flag, where $\cS$ is a set of {\em stopping} states defined later.
Clearly, for $K \geq 2$, MCLV-$K$ is distinct from CD-$K$, as the MC of MCLV-$K$ may stop before performing all $K$ steps. 

Analyzing CD-$K$ through our Las Vegas transformation, it is clear that CD-$K$ has an unintended inspection paradox bias that can be corrected to further improve the RBM learning.
Using the {\em reached-stopping-set} flag of MCLV-$K$, we design a new gradient estimator, denoted \LVname (\LV), that empirically gives significantly better parameter estimates than standard CD-1 and CD-10 over the MNIST dataset according to the model's likelihood. MNIST is used in our experiments due to the long history of RBM development over this dataset.

\paragraph{Contributions.}
We claim the following contributions:
{\bf (1)} We introduce Markov Chain Las Vegas (MCLV-$K$). We show MCLV-$K$ gives finite-sample unbiased and asymptotically consistent estimates of a variety of statistics of RBMs; further, we give two convergence bounds. We also show how to theoretically and empirically reduce the MCLV-$K$ random running times using the training examples.
{\bf (2)} We show how MCLV-$K$ can be used to design new ways to train Restricted Boltzmann Machines; we use MCLV-$K$ to propose a novel RBM gradient estimator, \LVname (\LV), which our empirical results show (for $K \in \{1,3,10\}$) improves parameter estimates of RBM over CD-1 and CD-10, over the MNIST dataset.

\section{MCLV-$K$ Estimation with Statistical Guarantees}
In what follows we introduce some of the definitions used throughout our paper.
We introduce the concept of a tour (a MC which returns to the same state) and show that the return probability can be increased by collapsing a set of stopping states into a single state in Definition~\ref{d:SCMC}. 
The MC stops when it either reaches $K$ steps or one of the states in the stopping set.
Corollary~\ref{c:Phi} describes how this collapsing can be performed while preserving the statistical properties of the MC.

Theorem~\ref{t:Ft} introduces the MCLV-$K$ estimator  (that, among others, can estimate the partition function) and proves it is consistent, giving error bounds. And Theorem~\ref{t:geotail} shows that this estimator is also finite-running-time unbiased.
The first results provides unbiased estimates of the partition function, and generalize these unbiased estimates to a broad family of functions. 

{\em The reader only interested in RBM gradient estimates can safely skip to the next section on training RBMs, after reading the preliminaries and the definition of the RBM stopping sets.}

\paragraph{Preliminaries.}
We define state $\bx = (\bv,\bh)$  to consist of a visible vector, $\bv \in V$, and a hidden vector $\bh \in H$, where $H$ and $V$ are the set of all hidden and visible states, respectively.
Let $\Phi(\bW)$ be an irreducible Markov chain with steady state $p(\bx;\bW)$ over the states $\Omega \coloneqq V \times H$.
A MC is irreducible if all states communicate, which is trivially true for RBMs since the co-domain of the logistic function is $(0,1)$ for any input in $\mathbb{R}$.
If the Markov chain $\Phi(\bW)$ starts in equilibrium (or runs until equilibrium), the next transition gives us {\em one} independent sample $\bx$ from the steady state $p(\bx;\bW)$.
The set $\{\bv_n\}_{n=1}^N$ denotes the $N$ visible examples of the training data.
We often use $(\cdot)$, as in $g(\cdot)$, to denote that the statement over $g$ is true for any valid input value.

RBMs can be trained by optimizing its parameters $\bW$ in order to maximize the likelihood of the training data. Taking partial derivatives with respect to the weights results in a surprisingly simple update rule for $\bW$:
\begin{equation}
\begin{split}
  \frac{1}{N} \sum_{n=1}^N & \frac{ \partial \log (\sum_\bh p(\bx = (\bv_n,\bh) ; \bW))}{\partial \bW}  \\
  =  \sum_{\bh \in H} &  \Big( \frac{1}{N} \sum_{n=1}^N p(\bh|\bv_n  ; \bW ) \bv_{n}\bh^{\tT}  \\
  &- \sum_{\bv \in V} p((\bv,\bh) ; \bW ) \bv\bh^{\tT} \Big)  \\
  = \frac{1}{N} &\sum_{n=1}^N \bv_{n} \EE_\bW[\bh | \bv_n]^{\tT} - \EE_\bW[\bv \bh^{\tT}] ,  \label{eq:grad}
  \end{split}
 \end{equation}
where the l.h.s.\ term of eq.~\eqref{eq:grad} (also called \textit{positive statistics}) is easily calculated from the training data. However, the r.h.s.\ term of eq.~\eqref{eq:grad} (\textit{negative statistics}) corresponds to the gradient of the partition function $Z(\bW)$, which is generally intractable to compute. More specifically, computing $E[\bv \bh^T]$ requires collecting model statistics $p(\bv,\bh)$, either by running the MCMC Markov chain $\Phi(\bW)$ to equilibrium from any starting state or by direct computation of the expected value if we know the partition function $Z(\bW)$. 

If the Markov chain $\Phi(\bW)$ is not run until equilibrium the gradient estimates have an unknown bias. 
In what follows we use Markov chain tours to take care of this bias.

\paragraph{Tours and Stopping Sets.}
Define a {\em tour} to be a sequence of $\xi$ steps of the Markov chain $(\bX(1),\ldots, \bX(\xi))$ s.t.\ the state of the $(\xi+1)$-st step is the same as the starting state, i.e., $\bX(1) = \bX(\xi+1)$.
Let $\Xi^{(r)} = (\bX^{(r)}(1),\ldots, \bX^{(r)}(\xi^{(r)}))$ denote the $r$-th tour, $r \geq 1$. The strong Markov property guarantees that if $s \neq r$, the sequences  $\Xi^{(r)}$ are independent of $\Xi^{(s)}$.
This independence guarantees that both $\Xi^{(r)}$ and $\Xi^{(s)}$ are sample paths obtained from the equilibrium distribution of the Markov chain. 
We will later use this property to obtain unbiased estimators of the partition function.

However, as is, tours are not a practical concept for RBMs because in such a large state space $\Omega$, the tour is unlikely to return to the same starting state.
We will, however, use a Markov chain property common to Metropolis-Hastings and Gibbs sampling Markov chains to significantly increase the probability of return by collapsing a large 
number of states into a single state.

\paragraph{RBM stopping set.} 
Our stopping set $\cS$ uses sampled hidden states from the training data, $\cH_N^{(m)} = \{\bh^{(1)}_n,\ldots,\bh^{(m)}_n : \bh_n \sim p(\bh | \bv_n ; \bW)\}$,
where $\{\bv_n\}_{n=1}^N$ is the training data.
Often we will use $m=1$, but we can change the size of $\cH_N^{(m)}$ by changing $m$. 
The stopping set contains all hidden states in $\cH_N^{(m)}$ and all possible visible states 
\begin{equation}\label{eq:S}
\cS^{(m)}_\text{HN} = \bigcup_{\bh \in \cH_N^{(m)}, \bv \in V} \{(\bv, \bh)\},
\end{equation}
and $p(\bh | \bv ; \bW)$ is the conditional probability of $\bh$ given $\bv$ using model parameters $\bW$. 
Most of our theoretical results apply to any stopping set that is a proper subset of the state space, $\cS \subset \Omega$.
{In practice, note that we {\em do not} store $\cS^{(m)}_\text{HN}$ in memory, rather we {\em only keep $\cH_N^{(m)}$ in memory}, as reaching a hidden state in $\cH_N^{(m)}$ is enough to guarantee we need to stop. This requires only $O(m N)$ space, where $N$ is the number of training observations.}
\begin{definition}[Stopping-set-Collapsed MC]\label{d:SCMC}
Consider an arbitrary stopping set $\cS \subset \Omega$.
A state-collapsed MC is a transformation of MC $\Phi(\bW)$ with state space $\Omega$, into a new MC $\Phi'(\bW)$  with state space $\Omega' = \Omega \backslash \cS \cup \{\bS\}$, where $\bS$ is a new state formed by collapsing all the states in $\cS$. 
The transition probabilities between states $\Omega' \cap \Omega$ are the same as in $\Phi(\bW)$.
The transition probabilities from $\bS$ to states in $\Omega' \backslash \{\bS\}$ are
\[
p_{\Phi'}(\bS,\bx) = \frac{\sum_{\by \in \cS} e^{-E(\by; \bW)} p_{\Phi}(\by,\bx) }{Z_\cS(\bW)} , \quad \forall \bx \in \Omega' \backslash \{\bS\},
\]
where $Z_\cS(\bW) = \sum_{\by \in \cS} e^{-E(\by; \bW)}$, and $p_a$ indicates the probability transition matrix of MC $a$.

The transitions from states $\Omega' \backslash \{\bS\}$ to state $\bS$ are
\[
p_{\Phi'}(\bx,\bS) = \sum_{\by \in \cS} p_{\Phi}(\bx, \by) , \quad \forall \bx \in \Omega' \backslash \{\bS\}.
\]
\end{definition}
It is important to distinguish the MC in Definition~\ref{d:SCMC} from general MC state aggregation methods such as lumpability~\cite{buchholz1994exact} and interactive aggregation-disaggregation methods~\cite{stewart1994introduction}.
In the following corollary, we see that the MC in Definition~\ref{d:SCMC} affects the steady state, unlike general MC aggregation methods that leave the steady state undisturbed.
Thankfully, later we will be able to correct the distortion imposed by Definition~\ref{d:SCMC} because we know the steady state distribution of the states inside $\cS$ up to a normalizing constant.
\begin{corollary}[Simulating $\Phi'(\bW)$ from $\Phi(\bW)$]\label{c:Phi}
For any MC $\Phi(\bW)$ resulting from standard Gibbs sampling or Metropolis-Hastings (MH) MCMCs,
we can cheaply simulate the transitions in and out of $\bS$ of $Definition~\ref{d:SCMC}$ by: (a) $p_{\Phi'}(\bS,\bx)$, we first sample a state $\by$ with replacement from $\cS$ with probability $e^{-E(\by; \bW)}/Z_\bS(\bW)$ and then perform a transition $p_{\Phi}(\by,\bx)$; (b) $p_{\Phi'}(\bx,\bS)$ is also simulated by performing a transition $p_{\Phi}(\bx, \by)$, and stopping the MC if $\by \in \cS$.
The simulated $\Phi'(\bW)$ is ergodic and time-reversible.
\end{corollary}
The proof is in the appendix.
It follows from the fact that $\Phi(\bW)$ is the MC of Gibbs sampling and MH and, thus, time-reversible~\cite{aldous2002reversible}.
Time reversibility imposes a set of necessary and sufficient conditions in the form of detailed balance equations~\cite[Theorem 6.5.2]{gallager2013stochastic}. A little algebra shows that the sampling procedure in Corollary~\ref{c:Phi} using $\Phi(\bW)$ is stochastically equivalent to $\Phi'(\bW)$.

\subsection{MCLV-$K$ Estimator}\label{s:MCLVK}
Following Corollary~\ref{c:Phi}, a tour starts by sampling the initial tour state $\bx$ and stopping when the tour reaches the stopping set $\cS$.
We now want to truncate all return times of tours greater than some value $K \geq 1$, i.e., we will only observe the complete $r$-th tour $(\bx,\bX^{(r)}(2),\ldots, \bX^{(r)}(\xi^{(r)}))$ if $\xi^{(r)} \leq K$. Otherwise, we observe only the first $K$ states of the tour: $(\bx,\bX^{(r)}(2),\ldots, \bX^{(r)}(K))$.
The {\em $\cS$-stopped} flag for tour $r$ is {\bf true} if $\xi^{(r)} \leq K$, otherwise it is {\bf false}.

\begin{lemma}[Perfect sampling of tours]\label{c:perfect}
Let $$\cC_k = \{(\bx,\bX^{(i)}(2),\ldots, \bX^{(i)}(k))\}_{i}$$ be a set of tours of length $k \leq K$, with $\bx$ sampled from $\cS$ according to some distribution.

Then, there exists a distribution $G_k$ such that the random variables
\begin{equation*}
\cG_k \equiv \{g(\sigma) : \: \forall \sigma \in \cC_k\}
\end{equation*}
are i.i.d.\ samples of $G_k$, with $g$ defined over the appropriate $\sigma$-algebra (e.g., $k$ RBM states) with $\Vert g(\cdot) \Vert_1 \leq \infty$. 

Moreover, if we perform $M$ tours, these tours finish in finite time and $\{\xi^{(r)}\}_{r =1}^{M}$ is an i.i.d.\ sequence with a well-defined probability distribution $p(\xi^{(\cdot)} = k)$.
\end{lemma}

The Las Vegas parallel is observed when we notice that any MCMC metric can be perfectly sampled from the tours. 
The tour lengths are sampled from a distribution $p(\xi^{(\cdot)} = k)$.
And, for any given tour length $k$, the metric of interest $g$ is perfectly sampled from $G_k$.
The maximum tour length $K$ only cuts off the tail of $p(\xi^{(\cdot)} = k)$ beyond $k > K$, which allows us to bound the sampling error.

\begin{theorem}[MCLV-$K$ RBM Estimator]\label{t:Ft}
Let $p(\bx ; \bW)$, $E(\bx;\bW)$, and $Z(\bW)$ be as described in  eq.\eqref{q:energy}.
Let 
\begin{equation} \label{eq:F}
F(\bW,f) = Z(\bW) \sum_{\bx \in \Omega} f(\bx) p(\bx; \bW),
\end{equation}
where $f:\Omega \to \mathbb{R}^n$, $n \geq 1$, $\Vert f(\cdot) \Vert_1 < \infty$, and $\Vert \cdot \Vert_1$ is the $l_1$ norm.
Let $\Phi(\bW)$ be a time-reversible MC with state space $\Omega$ and  steady state distribution $\{p(\bx ; \bW)\}_{\bx  \in \Omega}$.
Let $\cS \subset \Omega$ be a proper subset of the states of $\Phi(\bW)$.

Sample $\bx' \in \cS$ with probability $e^{-E(\bx'; \bW)}/Z_\cS(\bW)$ and let $(\bX^{(r)}(1)=\bx',\bX^{(r)}(2),\ldots, \bX^{(r)}(\xi^{(r)}))$ be a sequence of discrete states of the $r$-th $\cS$-stopped tour, where we stop the tour if one of two conditions are met: (a) we have reached $K$ steps, or (b) when we reach any state in $\cS$, i.e., $\bX^{(r)}(\xi+1) \in \cS$.
Then, for $R \geq 1$ tours, let $\cC_k^{(R)}$ be the set of finished tours in $k \leq K$ steps, (as defined in Corollary~\ref{c:perfect}).
For the sake of simplicity, we henceforth refer to $\cC_k^{(R)}$ simply as $\cC_k$.
The estimator
\begin{equation}\label{eq:Ft}
\begin{split}
&\hat{F}^{(K,R)}(\bW,f) = \frac{1}{\sum_{k=1}^K |\cC_k|} \sum_{\by \in \cS} e^{-E(\by; \bW)} \\
&\quad \times \sum_{k=1}^K \: \sum_{(\bX(1),\bX(2),\ldots, \bX(k))\in \cC_k} \: \sum_{h=1}^{k} f(\bX(h))
\end{split}
\end{equation}
is an estimate of $F(\bW,f)$ in eq.~\eqref{eq:F} with a bias upper bounded by $B \cdot (E[\xi] - \sum_{k=1}^{K-1} p(\xi > k))$, where $p(\xi > k)$ is the probability that a tour has length greater than $k$ and $B \geq \sup_{\bx \in \Omega} \Vert f(\bx) \Vert_1$.
\end{theorem}
Theorem~\ref{t:Ft} gives a basic estimator from the MCMC tours.
The gradient estimates will be explicitly derived in the next section.
In our experiments we show how to estimate $p(\xi > k)$.
For the partition function and gradient estimates, it is also trivial to obtain a bound on $B$ using the RBM weights $\bW$~\cite{Bengio2009}.

\begin{theorem}[Geometrically Decaying Tour Length Tails]\label{t:geotail}
Let $p(\xi > k)$ be the probability that a tour has length greater than $k$.
If there exists a constant $\epsilon > 0$ s.t.
$
\inf_{\bx \in \Omega \backslash \cS} \sum_{\by \in \cS} p_\Phi (\bx,\by) \geq \epsilon
$ 
then, there exists $0 < \alpha < 1$,
$
\log p(\xi > k) = k \log \alpha + o(k), 
$
i.e., $\xi$ has a geometrically decaying tail.
\end{theorem}
Theorem~\ref{t:geotail} shows conditions of a geometric decay in the tail of $p(\xi > k)$. And in practice it means that tours cannot be ``heavy tail'' long and, thus, making the bound in Theorem~\ref{t:Ft} tighter.

\subsection{MCLV-$K$ Finite-Sample Unbiasedness}\label{s:unbiasZ}
In what follows we dynamically increase $K$ until the MC reaches a state in the stopping set. 

The following theorem shows that this procedure gives unbiased estimates of $F(\bW,f)$.
\begin{theorem}[Unbiased Partition-scaled Function Estimates by Dynamic Adaptation of $K$]
\label{t:unbiasSuperF}
Consider the estimator in Theorem~\ref{t:Ft}
and let us dynamically grow $K$ (denoted $K_\text{dyn}$) until the MC reaches a stopping state in $\cS$.
Then, for $R \geq 1$ tours,
\begin{equation}\label{e:FprimeSuper}
\begin{split}
&\EE[\hat{F}^{(K_\text{dyn},R)}(\bW,f)]  = F(\bW,f) ,
\end{split}
\end{equation}
is an {\bf unbiased estimator} and the estimator is consistent, i.e., almost surely
$\lim_{R \to\infty} \hat{F}^{(K_\text{dyn},R)}(\bW,f) = F(\bW,f)$, and $K_\text{dyn}$ is finite.

Moreover, for $\epsilon > 0$,
$$
p\left( \left\Vert \hat{F}^{(K_\text{dyn},R)}(\bW,f) - F(\bW,f) \right\Vert_1  \geq \epsilon \right) \leq \alpha_{R,Z_\cS(\bW)} ,
$$ 
where, $R$ is the number of tours, $\alpha_{R,Z_\cS(\bW)} = \frac{B^2}{\epsilon^2 R} \left(\frac{(Z(\bW))^2}{(Z_\cS(\bW))^2 \delta}+1 \right)$, $B \geq \sup_{\bx \in \Omega} \Vert f(\bx) \Vert_1$ is an upper bound on the absolute value of $f(\cdot)$ over the state space $\Omega$, $\delta$ is the spectral gap of the transition probability matrix of $\Phi(\bW)$.

\end{theorem}
\begin{corollary}[Unbiased Partition Function Estimation]
\label{c:unbiasZ}
Let $f_1(x) = 1$, then
\[
\EE[\hat{F}^{(K_\text{dyn},R)}(\bW,f_1)]  = Z(\bW) ,
\]
is an unbiased estimator of the partition function.
\end{corollary}

\section{Training Restricted Boltzmann Machines}\label{s:RBM}
In what follows we explore the connections between MCLV-$K$ and learning RBMs using MCMC methods.
First, we show how MCLV-$K$ can provide a finite-sample unbiased and asymptotically consistent estimate of the direction
 of the RBM gradient.

\subsection{MCLV-$K$ Gradient Estimates}\label{s:CD}
In what follows we will provide an estimate of the gradient of the negative log-likelihood of RBMs using MCLV-$K$.
Our gradient will have a scaling factor but the gradient direction is the same as the original gradient:
\begin{align*}
\nabla_\bW \cL_Z &= 
 \frac{Z(\bW)}{Z_\cS(\bW)} \left( \frac{1}{N} \sum_{n=1}^N \bv_n^\mathsf{T} \EE_\bW[\bh | \bv_n] - \EE_\bW[\bv^\mathsf{T} \bh] \right).
\end{align*}
The scaling $Z(\bW)/Z_\cS(\bW)$ is constant given $\bW$. In our current implementation, we use Corollary~\ref{c:unbiasZ} to estimate $Z(\bW)/Z_\cS(\bW)$ and divide the gradient by it, compensating for the scaling at essentially no computational or memory cost.

\begin{corollary}[\LV-$K$: The \LVname Estimator]\label{c:lv}
Let $\Phi(\bW)$, $\cS$, $\bx$, the tour $(\bx,\bX^{(r)}(2),\ldots, \bX^{(r)}(\xi^{(r)}))$, $R$, $K$, and $\cC_k$ be as defined in Theorem~\ref{t:Ft}. 
Then, for a learning rate $\eta > 0$,
\begin{equation}\label{eq:LV}
\begin{split}
&\widehat{\nabla_\bW \cL_\text{\LV}}(K,R)  = \eta \Bigg( \frac{\widehat{\EE}[\xi]}{N} \sum_{n=1}^N  \frac{\partial E(\bx_n; \bW)}{\partial \bW} \\
& - \frac{ \sum_{k=1}^K \sum_{(\bX(1),\ldots, \bX(k)) \in \cC_k} \sum_{i=1}^{k} \frac{ \partial  E(\bX^{(r)}(k); \bW) }{\partial \bW}}{\sum_{k=1}^K | \cC_k |} \Bigg) ,
\end{split}
\end{equation}
is a consistent ($K,R \to \infty$) estimator of the energy-model gradient in eq.\eqref{eq:grad}, where $\widehat{\EE}[\xi] = \frac{\sum_{k=1}^K | \cC_k | k}{\sum_{k=1}^K | \cC_k |}$ is the empirical expectation of the tour lengths. 

Moreover, the contribution of a tour of length $k$ to the negative statistics of the gradient is proportional to 
$$P[\xi = k] \cdot k \cdot \EE[\partial  E(\tilde{\bX}_k; \bW)/\partial \bW ],$$
where $\tilde{\bX}_k$ is a random state of a tour of length $k$.
If the Markov chain $\Phi(\bW)$ satisfies the conditions of Theorem~\ref{t:geotail}, then $P[\xi = k] \cdot k = e^{-O(k)}$, so that {\em extremely long} tours do not influence the gradient.
\end{corollary}
\begin{corollary}[Unbiased Gradient Direction Estimator: \LV-$K_\text{dyn}$]\label{c:lvunb}
Consider the estimator in Corollary~\ref{c:lv}
and let us dynamically grow $K$ (denoted $K_\text{dyn}$) until the MC reaches a stopping state in $\cS$. Then,
\[
\EE\left[\widehat{\nabla_\bW \cL_\text{\LV}}(K_\text{dyn},R)\right]
\propto \nabla_\bW \cL_Z ,
\]
is an unbiased estimate of the RBM gradient direction.
\end{corollary}
The proofs of the two above corollaries follow directly from Theorems~\ref{t:Ft} and~\ref{t:unbiasSuperF}, respectively.

\subsection{Correspondence and Differences Between \LV-$\boldsymbol{K}$ and CD-$\boldsymbol{K}$} \label{s:TTCD}
In this section we explore a correspondence between \LV-${K}$ (proposed in Corollary~\ref{c:lv}) and CD-${K}$  to train RBMs.
We will also emphasize some differences that will give us some new insights into CD-$K$.
The correspondence is as follows:
{\bf (a)} consider a mini-batch of training examples $\{\bv_i\}_{i=1}^N$;
{\bf (b)} the stopping set is $\cS^{(m)}_{H\!N}$, described in eq.\eqref{eq:S};
{\bf (c)} the number of tours $R$ of \LV-$K$ is the number of training examples in the mini-batch $N$, i.e., $R=N$.

One can readily verify that the Gibbs sampling updates of \LV-$K$ and CD-$K$ are similar except for the following key differences:
{\bf (i)} \LV-$K$ starts at a state $\bx$ of $\cS^{(m)}_{H\!N}$ with probability proportional to $\exp(-E(\bx;\bW))$, CD-$K$ starts uniformly over the training examples. Thus, the negative phase of \LV-$K$ tends to push the model away from unbalanced probabilities over the training examples.
{\bf (ii)} at every Markov chain step, \LV-$K$ stops early if it has reached a state in $\cS^{(m)}_{H\!N}$, while CD-$K$ will always perform all $K$ steps. 
{\bf (iii)} the gradient estimates of \LV-$K$ use only the completed tours, while CD-$K$ uses all tours; 
{\bf (iv)} the gradient estimates of \LV-$K$ use all states visited by the MC during a tour, while CD-$K$ uses only the last visited state.

A long sequence of states visited by the CD-$K$ Gibbs sampler can be broken up into tours if the stopping state contains only the starting state.
Figure~\ref{f:tourbias} illustrates three MCMC runs starting at  visible states representing ``7'', ``3'', and ``4'', broken up into tours whenever the starting hidden state is sampled again.
Starting from visible state ``7'', CD-$K$ ignores the completed tour {\em Tour 1}, which \LV-$K$ uses for its gradient estimate; and CD-$K$ proceeds to use the state in the middle of {\em Tour A} for its gradient estimate.
CD-$K$ also uses a state in the incomplete {\em Tour 2}, which \LV-$K$ ignores as incomplete.
Finally, CD-$K$ ignores {\em Tour 3} and proceeds to use the state in the beginning of {\em Tour B} for its gradient estimate.

This means that, for $K \geq 2$, CD-$K$ is more likely to sample states from longer tour than shorter tours.
This bias is the inspection paradox~\cite{wilson1983inspection}.
Interestingly, this bias makes CD-$K$, $K \geq 2$, significantly different from CD-1, which has no such bias.
Note that \LV-$K$ has the opposite bias: it ignores tours longer than $K$; the bias of \LV is measurable (Theorem~\ref{t:Ft}) if we can estimate the average tour length.

\begin{figure}[ht!]
	\centering
 	\begin{minipage}[t]{.4\textwidth}
			\centering
			\includegraphics[height=1.4in]
			{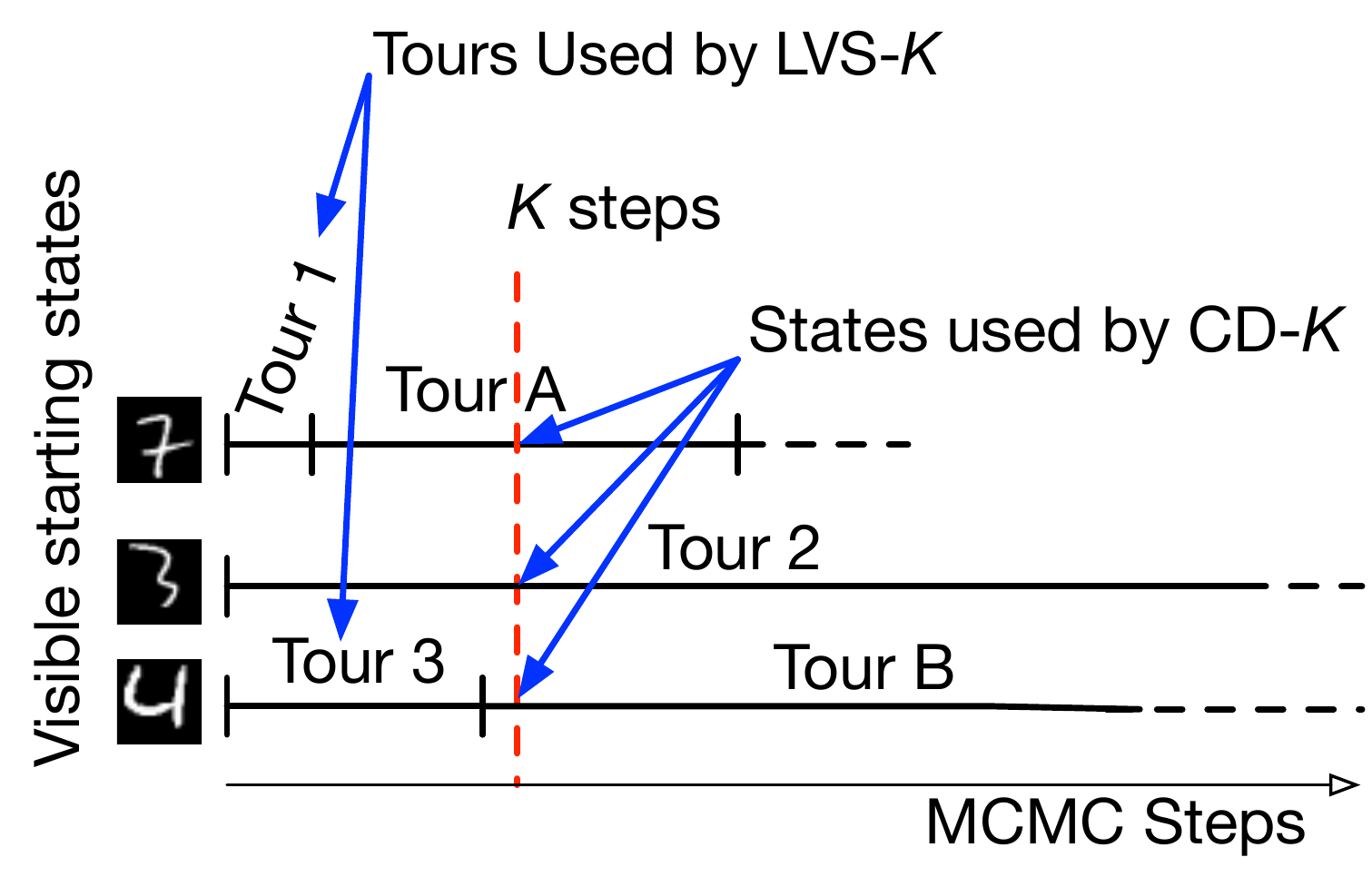}
			\caption{CD-$k$ bias towards longer tours for $k\geq 2$;}
			\label{f:tourbias}
 	\end{minipage}
\end{figure}

\subsection{Computational Complexity} \label{s:Complexity}

In this section we give the time and space complexities of LVS-$K$ and  CD-$K$. Let $|\bW|$ denote the number of elements in $\bW$, $= n_{V}*n_{H}$ and $n_{X}=n_{V}+n_{H}$.
In terms of space, LVS-$K$ needs $\cO(N \, m \, n_{X})$ space to store the $\cH^{(m)}_N$ which is $m$ times the requirement for CD-$K$.
At every epoch, LVS-$K$ samples a stopping set, which involves a matrix multiplication followed by algebraic operations over a matrix. 
The matrix multiplication which takes $\cO(N m |\bW|)$ upper bounds the time. 
Computing the free energies of the hidden state also takes the same time. 
Adding the states of the stopping set to a heap for easier sampling takes $\cO(N m)$ time and allows us to sample starting states for the tours in $\cO(N \log(N m))$.
Every Gibbs step is again bounded by the time taken for matrix multiplication which takes a total of $\cO(N |\bW| K)$ time.
Checking stopping set membership takes $\cO(N K n_{X})$ amortized time assuming that standard algorithms used by hash sets, e.g. MD5, take $\cO(n_{X})$ time to evaluate.
Computing the gradient and updating $\bW$ takes  $\cO(N |\bW|)$  time.

Therefore LVS-$K$ takes $\cO(N K n_{X} + N m |\bW|+ N K |\bW|) \equiv \cO(N |\bW| (m+ K))$ time, compared to CD-$K$ which takes $\cO(N K |\bW|)$. In the general case $m \in \cO(K)$, $\therefore$ the asymptotic complexity of CD-$K$ and LVS-$K$ are the same.

\section{Related Work}

\paragraph*{AIS, MC changes, and CD-$K$ extensions.}
Annealed Importance Sampling (AIS)~\cite{neal2005estimating,RadfordNeal2001} uses two distinct Markov transitions kernels and has been applied by Salakhutdinov and Murray~\cite{salakhutdinov2008quantitative} to obtain unbiased estimates of the partition function of an RBM. 
Like AIS, the Russian roulette pseudo-marginal likelihood is also a Markov chain modification to sample from the steady state distribution~\cite{lyne2015russian}.
These modifications cannot be readily applied to the original RBM Markov chain, nor they provide insights into the learning process.
MCLV-$K$ is a new tool that can be used from visual inspection of convergence to proposing new gradient estimators, as seen in our empirical results.

RBMs are powerful models~\cite{montufar2015discrete} and the analysis of CD-$K$ has a long history~\cite{Hinton2012}. A few past studies have focused on how CD-$K$ learns well RBMs~\cite{Carreira-Perpinan2005,Yuille2005}, have some fixable issues learning RBMs~\cite{Schulz2010,Fischer2014,Prats2014}, may approximate some objective function~\cite{Hinton2002}, or do not approximate any objective function~\cite{Bengio2009,Sutskever2010}. 
Orthogonally, Persistent Contrastive Divergence (PCD)~\cite{Tieleman2008} improves CD-$K$ in some problems by using the starting state of the CD-$K$ Markov chain as the end state of the previous training epoch (simulating a single sample path, assuming the MC does not change much between epochs, which is not always true~\cite{Schulz2010}). Clearly, PCD could be adapted as a MCLV method, which we see as future work.

The presence of training data is key to the practicality of MCLV.
Without training data, obtaining error bounds with MCLV can be prohibitively expensive. In the worst-case, there is no polynomial time algorithm that can estimate the probabilities of an RBM model within a constant factor~\cite{long2010restricted}, assuming P$\neq$NP.
But most real-world machine learning problems are {\em supposed} to be much easier than general MCMC results would have us believe. We are {\em given} a good hint of what should be a larger number of high-probability states in the steady state: the states containing the training examples. 
Unfortunately, vanilla MCMC methods do not incorporate this extra information to speed up convergence in a principled way.
We believe the lessons learned in this paper will be invaluable to design new classes of Markov chain methods tailored to machine learning applications.

\paragraph*{Las Vegas algorithms for Markov chain sampling.}
Perfect Sampling~\cite{Corcoran2002,fill1997interruptible,fill2000extension,propp1996exact,propp1998get,wilson2000layered} is an example of a Las Vegas algorithm for MCMC applications.
Unfortunately, energy-based models can easily reach trillions of states while perfect sampling methods rarely scale well unless some specific MC structure can be exploited.
We are unaware of clever CFTP constructions for  arbitrary energy-based models.

Mykland et al.~\cite{Mykland1995} with a few follow-up works first proposed the use of regeneration in the context of MCMC to estimate mixing times, however these techniques are mostly of theoretical interest~\cite{Baxendale2005,Gilks1998a,Hobert2002,Roberts1999} rather than of practical utility for energy-based models.
Path coupling is another alternative to estimate mixing times~\cite{Bubley1997}.
More recently,  path coupling was used to develop a theory of Ricci curvature for Markov chains~\cite{Ollivier2009}.
The connections between Ricci curvature estimation and MCLV-$K$ are worth exploring in future work.

\section{Empirical Results}

Our experiments use the MNIST dataset, which consists of 70,000 images of digits ranging from 0 to 9, each image having $28 \times 28$ pixels (a total of $784$ pixels per image), divided into 55,000 training examples and 15,000 test examples.
We use this dataset for historical reasons.
MNIST is arguably the most extensively studied dataset in RBM training, e.g. ~\cite{Hinton2002,Hinton2012,hinton2006fast,Carreira-Perpinan2005,Salakhutdinov2009,Tieleman2008,Bengio2009}.
Our goal is to show that MCLV-$K$ is able to give new insights into RBM training (and improved performance) even in a studied-to-death dataset such as MNIST. The experimental details of our empirical results are presented in the appendix.
We use \LV-$1$ to train the RBM model used in the following experiments (CD-$K$ tends to give very high probability to a few examples in the training data).
We observe little difference between \LV-$1$, \LV-$3$, and \LV-$10$ (for reasons that will be clear soon).

\paragraph{RBM learning.}%
\label{s:grad}
Our first set of empirical results compares  \LV-$K$, $K \in \{1,3,10\}$, CD-$K$, $K \in \{1,10\}$ and PCD-$K$, $K \in \{1,10\}$ by training an RBM using stochastic gradient descent, where the gradient estimates are computed using the respective methods.
We train RBMs with $n_{H} = 32$ hidden neurons for a total of 100 epochs (inclusive of 15 warm-up epochs of CD-$1$ for \LV-$K$), using a learning rate of $0.1$ which decays according to a Robbins-Monro schedule. Weight decay and momentum were not used. 
The initial $\bW$ weights are sampled uniformly from $U \left( \frac{-0.1}{\sqrt{n_V + n_H}}, \frac{0.1}{\sqrt{n_V + n_H}} \right)$, where $n_V$ and $n_H$ denote the number of visible and hidden neurons, respectively. Hidden biases are initialized as zero, while visible biases are initialized as $\log(p_i/(1-p_i))$ \cite{Hinton2012}, where $p_i$ is the empirical probability of feature $i$ being active.

The small number of hidden units is to enable us  
to evaluate the true performance: we compute the exact partition function of the trained RBM and calculate the average log-likelihood $\frac{1}{N} \sum_{n=1}^N \log p(\bv_n)$. All results are means calculated from $10$ executions. In all \LV-$K$ experiments we use $m=1$, for simplicity.
The negative log-likelihood of \LV-$K$, PCD-$K$ and CD-$K$ are presented in Table~\ref{tab:h32results}. 

Subsequently, in order to compare our results with those presented in ~\citeauthor{Tieleman2008} (\citeyear{Tieleman2008}), we train RBMs with $n_{H} = 25$ and initial learning rates between $10^{-4}$ and $1$. We observe that larger learning rates ($10^{-1}$ to $1$) are more appropriate for \LV-$K$, resulting in faster convergence and increased performance. Small rates (e.g. $10^{-4}$) cause tours to rarely finish, severely slowing down the training. On the other hand, CD-$K$ and PCD-$K$ fail to converge with learning rates slightly larger than $10^{-2}$. The results for this experiment, along with the best learning rates for each method, are presented in Table~\ref{tab:h25results}.


In conclusion, \LV-$K$ drastically (and paired t-test significantly) outperforms CD-$K$ and PCD-$K$ w.r.t.\ the log-likelihood in all settings, even \LV-$1$ performs significantly better than PCD-$10$. 
However we were unable to reproduce the likelihood of $\approx -130$ for PCD achieved by ~\citeauthor{Tieleman2008} (\citeyear{Tieleman2008}). 

\begin{table}[ht]
\centering
    \begin{tabular}{l|cc}
\textbf{Method}	&	\textbf{Training } &	\textbf{Testing} \\
\hline 
CD-$1$	&	-167.3 (2.7) &	-166.6 (2.8) \\
CD-$10$	&	-154.3 (3.3) &	-153.4 (3.3) \\
PCD-$1$	&	-153.0 (4.9) &	-152.1 (4.7) \\
PCD-$10$	&	-139.3 (3.2) &	-138.5 (3.3) \\
\textbf{\LV-1}	&	\textbf{-134.0 (1.0)} &	\textbf{-133.3 (1.0)} \\
\textbf{\LV-10}	&	\textbf{-133.3 (1.0)} &	\textbf{-132.6 (1.0)} \\
\textbf{\LV-3}	&	\textbf{-133.7 (0.8)} &	\textbf{-132.9 (0.7)} \\
    \end{tabular}
    \caption {(Higher is better) Average log-likelihood on the MNIST dataset using a RBM with $32$ hidden neurons. Results are means over 10 executions after 100 epochs.}
    \label{tab:h32results}
\end {table}

\begin{table}[ht]
\centering
    \begin{tabular}{ll|cc}
\textbf{Method}	& \textbf{Learning Rate}	&	\textbf{Training } &	\textbf{Testing} \\
\hline 
CD-$1$	&	0.01	&	-169.8 (2.6)	&	-169.0 (2.6)	\\
CD-$10$	&	0.01	&	-156.4 (0.5)	&	-155.6 (0.5)	\\
PCD-$1$	&	0.01	&	-147.8 (0.5)	&	-147.0 (0.5)	\\
PCD-$10$	&	0.01	&	-147.4 (0.5)	&	-146.7 (0.5)	\\
\textbf{\LV-1}	&	\textbf{0.1}	&	\textbf{-138.3 (1.3)}	&	\textbf{-137.5 (1.4)}	\\
\textbf{\LV-10}	&	\textbf{0.1}	&	\textbf{-138.1 (1.1)}	&	\textbf{-137.4 (1.2)}	\\
\textbf{\LV-3}	&	\textbf{0.1}	&	\textbf{-138.2 (1.0)}	&	\textbf{-137.5 (1.1)}	\\
    \end{tabular}
    \caption {(Higher is better) Average log-likelihood on the MNIST dataset using a RBM with $25$ hidden neurons. Results are means over 10 executions after 100 epochs, using appropriate learning rates for each method.}
    \label{tab:h25results}
\end {table}

\paragraph{Tours lengths and stopping state.}%
We now analyze the tour lengths as a function of: (a) $n_H$, the number of hidden units, and (b) the size of the stopping set $|\cS^{(m)}_{N\!H}|$, where $\cS^{(m)}_{N\!H}$ is built from the training data as defined in eq.\eqref{eq:S}. Note that the $r$-th tour ends at state $\bX^{(r)}(\xi) = \left(\bv^{(r)}(\xi),\bh^{(r)}(\xi) \right)$ whenever $\bX^{(r)}(\xi+1) \in \cS_{N\!H}$, and that the stopping criteria only truly depends on $\bh^{(r)}(\xi+1)$ since $\cS^{(m)}_{N\!H}$ contains all possible visible states.

Figure \ref{f.ccdf} shows the CCDF of the tour lengths for different values of $n_H$. 
Most tours are extremely short for RBMs with few hidden neurons (for $n_H = 16$, more than $99\%$ have length one), but significantly increase as we increase $n_H$ with a very heavy tail.
Thus, it is expected that we see little difference between \LV-$1$, \LV-$3$, and \LV-$10$.
Moreover, these heavy tails likely causes strong inspection-paradox biases for CD-$K$ in high-dimensional RBMs. 

Most importantly, Figure \ref{f.ccdf} shows that tours either return within one step or are unlikely to return for a very long time. A closer inspection at these one-step tours, shows that over 99\% of the cases have the hidden state being the starting state. {\em Thus, it seems that RBMs (even with few hidden neurons) are just memorizing the training data, not learning how to generate new digits.} We conjecture, however, that if our training could force the tours to stop at distinct hidden states, and requires the tours to be possibly longer (but not too-long), the RBM might be taught how to generate new digits. 

Using $n_H = 32$ hidden neurons, Figure \ref{f.SNccdf} shows the probability that a tour takes more than $k$ steps, as we increase the number of stopping states by setting the values of $m \in \{1,4,7\}$ in $\cS^{(m)}_{N\!H}$. We see that the probability of tours finishing in a single step increases as we add more states to the stopping state. 
Thus, increasing the stopping set size can significantly shorten the tours, which in turn improves the estimates of MCLV-$K$ and \LV-$K$, and is an avenue to ameliorate MCMC issues in high-dimensional RBMs.

\begin{figure}[ht!]
	\centering
	\begin{minipage}[t]{.47\textwidth}
		\begin{subfigure}[t]{0.95\textwidth}
			\centering               
            \includegraphics[width=0.95\textwidth]{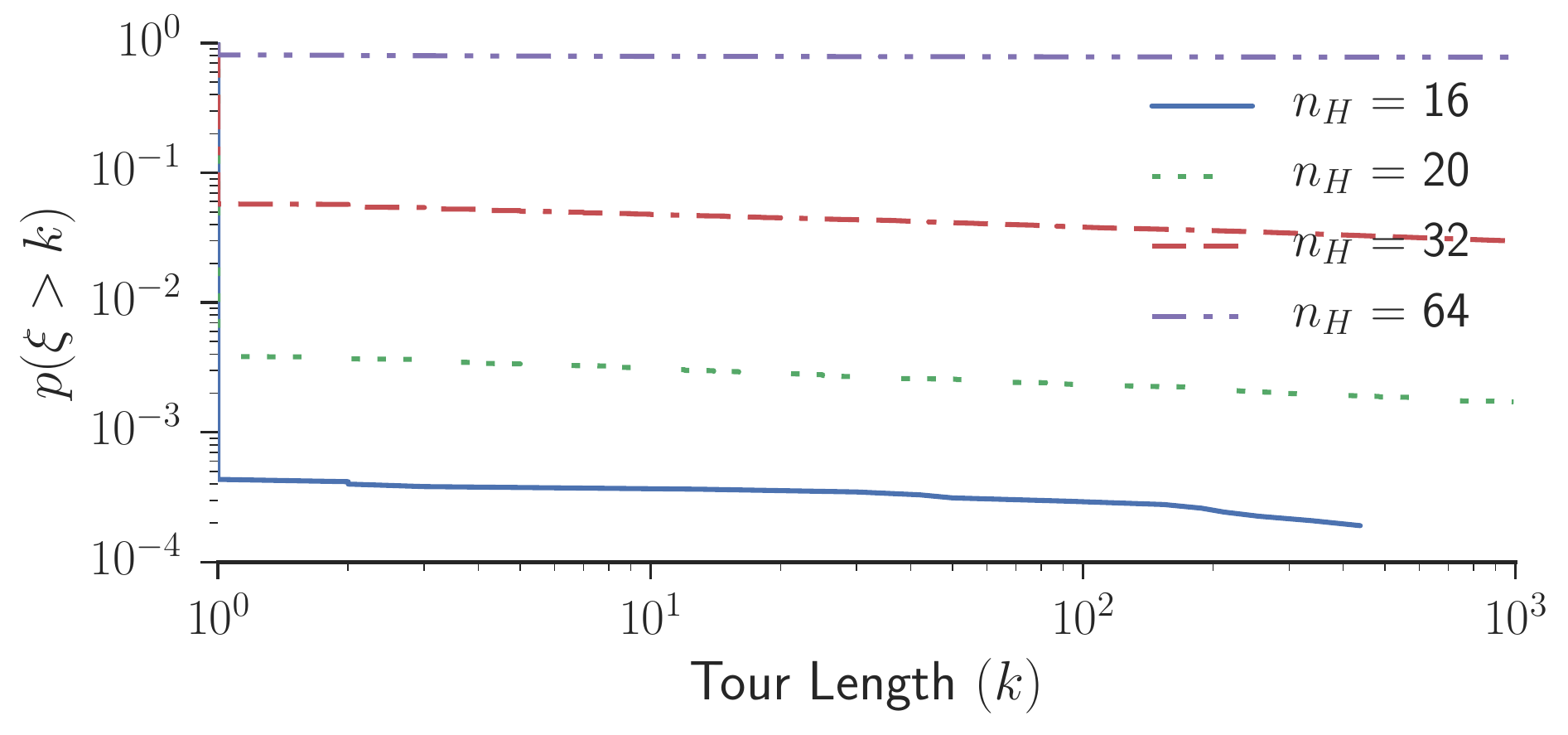}
			\caption{~}
			\label{f.ccdf}
		\end{subfigure}
\end{minipage}
    \begin{minipage}[t]{.47\textwidth}
		\begin{subfigure}[t]{0.95\textwidth}
			\centering
\includegraphics[width=0.95\textwidth]{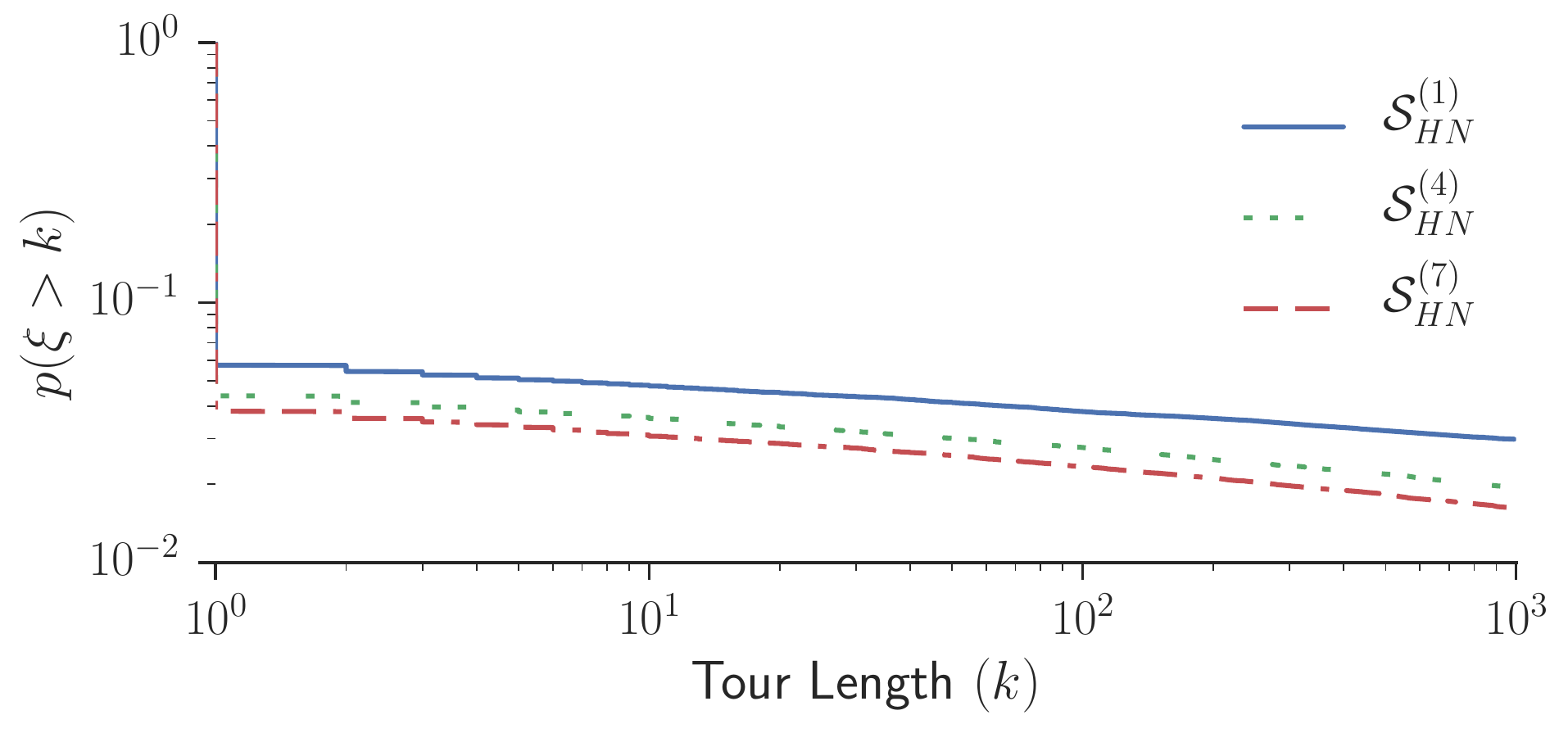}
			\caption{~}
			\label{f.SNccdf}
		\end{subfigure}
	\end{minipage}
	\begin{minipage}[t]{.47\textwidth}
		\begin{subfigure}[t]{0.95\textwidth}
			\centering		
            \includegraphics[width=0.95\textwidth]{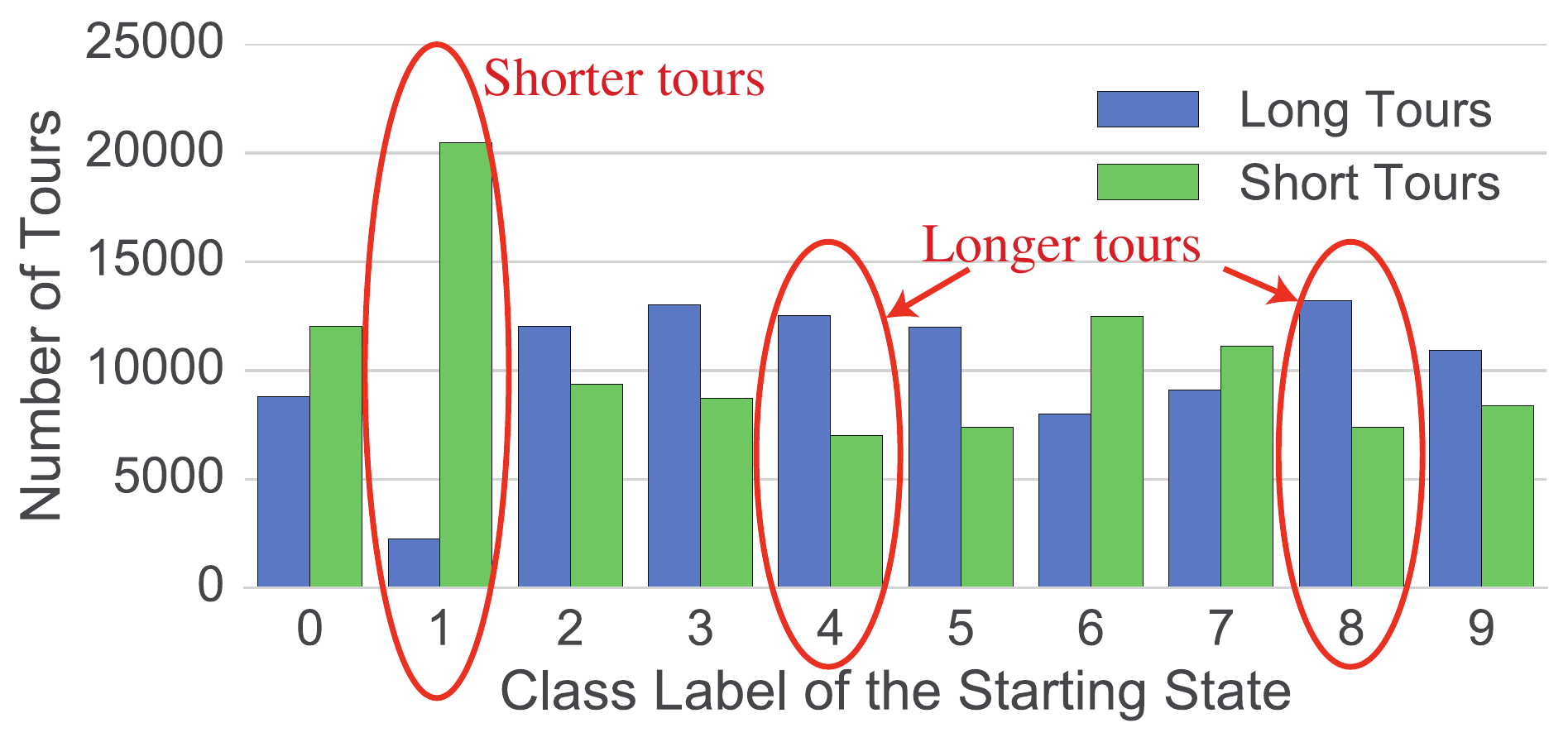}
			\caption{~}
\label{f.shortlongtourbar}
		\end{subfigure}
	\end{minipage}
	\caption{
		(a) Tour lengths CCDF for $n_H = \log_2 |H| \in \{16,20,32,64\}$ for LVS-1; 
		(b) Tour lengths CCDF variation for LVS-10 with  $n_{H} = 32$, using larger Stopping Sets;
        (c) Comparison of frequencies of short and long tours starting from labeled states on a trained RBM
	}
\end{figure}

\paragraph{Distribution modes of the learned RBM.}%
Overall, we may want to ask which digits (pictures) the model is learning to reproduce well.

Figure~\ref{f.shortlongtourbar} shows the length of the tours split by the type of digit starting the tour. Note that the RBM seems to learn well digits that are more consistent across the training data (e.g., numbers one and six) and have more trouble with digits that have more variability (e.g., numbers four and eight). 

As a visual inspection, Figure \ref{f.tours32} shows the next visible states of extremely short (length = 1) and long (unfinished after 99,999 steps) tours, for $n_H = 32$ hidden neurons. There is a clear relation between long tours and not-so-common examples in the training data. The first and third rows show the training examples; the next row shows their first visible state after one Gibbs sampling step. Note that the majority of the training examples are easy-to-recognize digits, with still recognizable digits after sampling.

The second part of Figure \ref{f.tours32} shows the training example and first visible samples of long tours. 
Note that the long tours tend to be digits that are either thicker (rarer in the data), or come in a not-so-standard shape than the digits in the first row.
Note that in half of the examples, their first Gibbs samples are not too similar to the original digit. 
This shows that the model is having trouble learning these less standard digits.
That long tours tend to start in odd-looking-examples, should help us better understand and avoid {\em fantasy particles} (visible states $\bv \in V$ that are not characteristic of the dataset but have high probability nonetheless~\cite{Tieleman2008}).

\begin{figure}[ht!]
	\centering
	\begin{minipage}[t]{.4\textwidth}
		\begin{subfigure}[t]{0.99\textwidth}
			\centering		
			\includegraphics[width=0.9\textwidth]{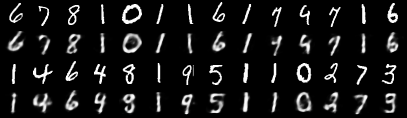}
			\caption{Short Tours}
		\end{subfigure}
	\end{minipage}
	\begin{minipage}[t]{.4\textwidth}
		\begin{subfigure}[t]{0.99\textwidth}
			\centering		
			\includegraphics[width=0.9\textwidth]{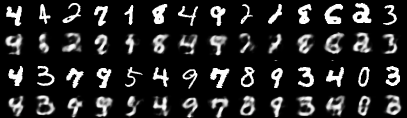}
			\caption{Long Tours}
		\end{subfigure}
	\end{minipage}
	\caption{
	\small Visible states of tours for $n_H = 32$ neurons. The first and third rows of each image show the visible states from the training data, whereas the second and fourth show the next visible state obtained through Gibbs Sampling}
	\label{f.tours32}
\end{figure}

\paragraph{Estimating the partition function.}
We use MCLV-$K_\text{dyn}$ to estimate the partition function $Z(\bW) = \sum_{\bv} \sum_{\bh} e^{-E( (\bv,\bh) ; \bW)}$ as specified in Corollary~\ref{c:unbiasZ} using an RBM with $n_H = 32$, so that we can easily compute the true partition function for comparison. We note that computing $Z_\cS(\bW)$ is fast as it is in the order of the number of training examples (as stated earlier). 

Following Corollary~\ref{c:unbiasZ}, we estimate $Z(\bW)$ with $\hat{F}^{(K_\text{dyn},R)}(\bW,f_1)$, with $f_1(x)=1$.
The average tour length in this example is estimated to be close to one (see Figure \ref{f:tourbias}).
Thus, $\hat{F}^{(K_\text{dyn},R)}(\bW,f_1) \approx Z_\cS(\bW) = 1.46 \times 10^{100}$ in this example.
In fact, $\hat{F}^{(K_\text{dyn},R)}(\bW,f_1)$ and the true partition function $Z(\bW)$ report the same value up to nearly machine precision (10-th decimal place).

\section{Conclusions}
This paper proposes a Las Vegas transformation of Markov Chain Monte Carlo (MCMC) for RBMs, denoted {\em Markov Chain Las Vegas} (MCLV). MCLV gives statistical guarantees in exchange for random running times.
Our empirical results show MCLV-$K$ is a powerful tool to learn and understand RBMs, with a gradient estimator \LV-$K$ that can better fit RBMs to the MNIST dataset than standard MCMC methods such as Contrastive Divergence (CD-$K$).

\section{Appendix}

\subsection{Proof of Corollary~\ref{c:Phi}}
\begin{proof}
Collapse the states of $\cS$ into a single state $\bS$ to form a state-collapsed MC $\Phi'(\bW)$, with transition probabilities given by Definition~\ref{d:SCMC}.
Let $(\bS,\bX^{(\cdot)}(2),\ldots, \bX^{(\cdot)}(\xi^{(\cdot)}))$ be a sequence of discrete states of the $r$-th tour of the state-collapsed MC $\Phi'(\bW)$.
Note that $\bS$ is the renewal state of the tour $\Xi^{(\cdot)}$, i.e., $\bX^{(\cdot)}(1) = \bS$.

The time reversibility of $\Phi(\bW)$ implies that 
$p(\bx ; \bW) p_\Phi (\bx, \by) = p(\by ; \bW) p_\Phi (\by, \bx)$, where $p_a$ indicates the probability transition matrix of MC $a$.
Let $Z_\bS(\bW) = \sum_{\by \in \cS} e^{-E(\by; \bW)}$.
We now show that $\Phi'(\bW)$ is time-reversible using the fact that the steady state distribution of $\Phi(\bW)$ is known up to a constant factor.
Thus, we ``guess'' the steady state distribution in $\Phi'(\bW)$  of $\bS$ as $p(\bS; \bW) =  Z_\bS(\bW)/Z(\bW)$ and verify that, because $\cS$ is a proper subset of $\Omega$, the balance equations of $\Phi'(\bW)$ are time reversible:
\begin{align*}
 p(\bS; \bW) p_{\Phi'}(\bS,\bx) &\coloneqq \frac{Z_\bS(\bW)}{Z(\bW)}  \sum_{\by \in \cS} \frac{e^{-E(\by; \bW)}}{Z_\bS(\bW)} p_{\Phi}(\by, \bx) \\
 &=   \sum_{\by \in \cS} p(\by;\bW) p_{\Phi}(\by, \bx)\\
 &=  \sum_{\by \in \cS} p(\bx;\bW) p_{\Phi}(\bx, \by) \qquad \text{see$^\dagger$}\\
 &=p(\bx; \bW) p_{\Phi'}(\bx,\bS),
\end{align*}
$\dagger$from the time reversibility of $\Phi(\bW)$.
Thus, all states $\bx \in \Omega' \backslash \{\bS\}$ in $\Phi'(\bW)$ have the same steady state distribution as in $\Phi(\bW)$: $p(\bx;\bW)$.
\end{proof}

\subsection{Proof of Lemma~\ref{c:perfect}}
\begin{lemma*}[Perfect sampling of tours]
Let $$\cC_k = \{(\bx,\bX^{(i)}(2),\ldots, \bX^{(i)}(k))\}_{i}$$ be a set of tours of length $k \leq K$, with $\bx$ sampled from $\cS$ according to some distribution.

Then, there exists a distribution $G_k$ such that the random variables
\begin{equation}
\cG_k \coloneqq \{g(\sigma) : \: \forall \sigma \in \cC_k\}
\end{equation}
are i.i.d.\ samples of $G_k$, with $g$ defined over the appropriate $\sigma$-algebra (e.g., $k$ RBM states) with $\Vert g(\cdot) \Vert_1 \leq \infty$. 

Moreover, if we perform $M$ tours, these tours finish in finite time and $\{\xi^{(r)}\}_{r =1}^{M}$ is an i.i.d.\ sequence with a well-defined probability distribution $p(\xi^{(\cdot)} = k)$.
\end{lemma*}
\begin{proof}
Consider an infinite run of the MCMC $\Phi'(\bW)$: $\bX(1),\bX(2),\ldots$, starting at state $\bX(1) = \bS$.
Divide this infinite run into tours, the longest segments of consecutive states that start at state $\bS$ but do not contain $\bS$ in any other states in the segment.
Let $\xi^{(r)}$ be the length of the $r$-th tour.
Because $\Phi'(\bW)$ is an irreducible Markov chain, it is positive recurrent~\cite[Theorem 6.3.8]{gallager2013stochastic}, and we can use Kac's theorem~\cite[Theorem 10.2.2]{meyn2012markov} to assert that $\EE[\xi^{(\cdot)}] < \infty$, which also implies $\xi^{(\cdot)} < \infty$ almost surely (i.e., except for a set of measure zero).
Define $R_{r+1} = R_r + \xi^{(r+1)}$, with $R_0 = 0$.
Define
\begin{equation*}
\begin{split}
	\cG_k = \{g(\bX(R_{r-1}),\ldots,\bX(R_{r}-1)) :\\
			 \,\,\,\, r=1,\ldots,M ,\xi^{(r)} = k\},
\end{split}
\end{equation*}
with $M > 1$.
By the strong Markov property, there exists a distribution $G_k$ such that $\cG_k$ is an iid sequence from $G_k$.
Note that  $\{\xi^{(r)}\}_{r =1}^{M}$ is also iid.
Further, note that by Corollary~\ref{c:Phi} we can equivalently consider the MC $\Phi(\bW)$, starting at state $\bx$ sampled from the stopping set $\cS$, which concludes the proof.
\end{proof}

\subsection{Proof of Theorems~\ref{t:Ft} and~\ref{t:unbiasSuperF}}
\begin{proof}
For simplicity, in what follows we combine the proofs of Theorems~\ref{t:Ft} and~\ref{t:unbiasSuperF}, specializing on each case when necessary.
Define for all $r \geq 0$, $R_{r+1} = R_r + \xi^{(r+1)}$ and for $t \geq 0$, $N(t) = \argmax_r \one_{\{R_{r-1} < t\}}$, with $R_0 = 0$.
$N(t)$ counts how many of the tours in the sequence $\{\xi^{(r)}\}_{r \geq 1}$ are needed to add up to the largest number smaller than $t$.
Let 
$$
Y^{(r)}_K = \one\{\xi^{(r)} \leq K\} \sum_{t=R_{r-1}+1}^{R_{r}}  f(\bX^{(r)}(t-R_{r-1})).
$$ 
By Lemma~\ref{c:perfect}, both $\{Y^{(r)}_K\}_{r \geq 1}$ and $\{\xi^{(r)}\}_{r \geq 1}$ are iid sequences.
Also, even in the case $K \to \infty$,
\begin{equation*}
\begin{split}
	\EE[\Vert Y^{(\cdot)}_K \Vert_1] & \leq  \EE[\sup_\bx \xi^{(\cdot)} \Vert f(\bx) \Vert_1] = \EE[\xi^{(\cdot)}]  \sup_\bx  \Vert f(\bx) \Vert_1 \\ 
							&< \infty,
\end{split}
\end{equation*}
as by definition $\Vert f(\cdot)\Vert_1 < \infty$ and we know $\EE[\xi^{(\cdot)}] < \infty$ (see Lemma~\ref{c:perfect}).
The Renewal-Reward Theorem~\cite[Theorem 4.2]{bremaud2013markov} yields, $r \geq 1 $
\begin{equation}\label{e:RR}
\lim_{t \to \infty} \frac{\sum_{r=1}^{N(t)} Y_K^{(r)}}{t} = \frac{\EE\left[\one_{\{\xi^{(\cdot)} \leq K\}} \sum_{k=1}^{\xi^{(\cdot)}}  f(\bX^{(\cdot)}(k))\right]}{\EE[\xi^{(\cdot)}]} .
\end{equation}
Note that, 
\begin{align*}
& \frac{\sum_{r=1}^{N(t)} Y_K^{(r)}}{t} = \\
&\,\,\,\,\, \frac{\one_{\{\xi^{(N(t'))} \leq K\}} \sum_{t'=1}^{R_{N(t)}} f(\bX^{(N(t'))}(t'-R_{N(t')-1}))}{t} \\ 
&= \frac{\sum_{t'=1}^{R_{N(t)}} \one_{\{\xi^{(N(t'))} \leq K\}} f(\bX^{(N(t'))}(t' - R_{N(t')-1}))}{R_{N(t)}} \\
&\,\,\,\,\,\,\,\,\,\,\,\,\,\,\,\,\,\,\,\, \cdot \frac{R_{N(t)}}{t} .
\end{align*}
Most importantly, $\lim_{t \to \infty} \frac{R_{N(t)}}{t}  = 1$, as by definition $R_{N(t)} + \xi^{(N(t)+1)} > t$, and  $\lim_{t\to\infty} \frac{\xi^{(N(t)+1)}}{t} \overset{}{=} 0$, otherwise an infinitely large $\xi^{(N(t)+1)}$ would have non-zero measure, contradicting $\EE[\xi^{(N(t)+1)}] < \infty$.
This yields,
\begin{align*}
&\lim_{t \to \infty} \frac{\sum_{r=1}^{N(t)} Y_K^{(r)}}{t} \\
&=\lim_{t \to \infty} \frac{\sum_{t'=1}^{R_{N(t)}} \one_{\{\xi^{(N(t'))} \leq K\}} f(\bX^{(N(t'))}(t' - R_{N(t')-1}))}{R_{N(t)}}.
\end{align*}
\paragraph{(Theorem~\ref{t:unbiasSuperF}) The case of $K_\text{dyn}$:}
As $K_\text{dyn}$ is finite almost surely (see proof of Lemma~\ref{c:perfect}), it makes the condition $\one_{\{\xi^{(N(t'))} \leq K_\text{dyn}\}} \coloneqq 1$.
Note that the sequence $\{\bX^{(N(t'))}(t' - R_{N(t')-1})\}_{t'=1}^{R_{N(t)}}$ is just a single sample path of our MC starting at state $\bx'$, taking $R_{N(t)}$ steps.
As our MC is irreducible and time-reversible, there is solution and the solution is unique~\cite[Theorem 6.3.8]{gallager2013stochastic}, and thus we can use the ergodic theorem to show
\begin{equation*}
\begin{split}
	\lim_{t \to \infty} \frac{\sum_{t'=1}^{R_{N(t)}} f(\bX^{(N(t'))}(t' - R_{N(t')-1}))}{R_{N(t)}} = \\ 
										\sum_\bx f(\bx) p(\bx ; \bW),
\end{split}
\end{equation*}
and substituting the above equation in \eqref{e:RR}, yields
\begin{equation}\label{e:Zf}
\EE\left[\sum_{k=1}^{\xi^{(\cdot)}} f(\bX^{(\cdot)}(k))\right] = \EE[\xi^{(\cdot)}] \sum_\bx f(\bx) p(\bx ; \bW) ,\quad r \geq 1.
\end{equation}
Finally, by Kac's theorem~\cite[Theorem 10.2.2]{meyn2012markov}, 
\begin{equation}\label{e:xiinv}
\EE[\xi^{(\cdot)}] = \frac{1}{p(\bx';\bW)} = \frac{Z(\bW)}{e^{-E(\bx'; \bW)}},
\end{equation}
as $p(\bx';\bW)$ is the steady state probability of visiting state $\bx'$.
Replacing \eqref{e:xiinv} into \eqref{e:Zf} and multiplying it by $e^{-E(\bx'; \bW)}$ on both sides concludes the unbiasedness proof.
Thus, if $\hat{F}^{(K_\text{dyn})}_r(\bW,f)$ denotes the estimator $\hat{F}$ in eq.\eqref{eq:Ft} applied to only a single tour $r  = 1,\ldots,R$.
Then, $\EE[F^{(K_\text{dyn})}_r(\bW,f)] = F(\bW,f)$ and the sequence $\{F^{(K_\text{dyn})}_r(\bW,f)\}_{r\geq 1}$ is trivially iid by the strong Markov property.
This iid sequence guarantees the following convergence properties.

Error bound:
Note that $ \sum_{k=1}^{\xi^{(r)}} \frac{ \partial  E(\bX^{(r)}(k); \bW) }{\partial \bW}$ is upper bounded by $\xi^{(r)} B$. 
As $\Phi(\bW)$ is time-reversible, it is equivalent to a random walk on a weighted graph. 
Thus, Lemma 2(i) of Avrachenkov et al.~\cite{avrachenkov2016inference} applies with $Z(\bW) = 2 d_\text{tot}$, $Z_\cS(\bW) = d_{\cS_n}$, and we have 
$$\text{var}(\hat{F}^{(K_\text{dyn})}_1(\bW)) \leq B^2 \left((Z(\bW))^2/(Z_\cS(\bW) \delta) + 1\right).$$
By the strong Markov property the tours are independent, thus, $\text{var}(\hat{F}^{(K_\text{dyn},R)}(\bW)) = \text{var}(\hat{F}^{(K_\text{dyn})}_1(\bW))/R$ by the Bienaym\'e formula.
And we have already shown that the estimate of $\hat{F}^{(K_\text{dyn})}_1(\bW)$ is unbiased.
Finally, we obtain the bound through the application of Chebyshev's inequality.

\paragraph{(Theorem~\ref{t:Ft}) The case of $K$:}
From above, $\hat{F}^{(K_\text{dyn},R)}(\bW,f)$ is an unbiased estimate of  $F(\bW,f)$.
The remaining of the proof is straightforward. The tours are independent. Thus $\EE[\hat{F}^{(K,R)}(\bW)] = \EE[\hat{F}^{(K,1)}(\bW)]$.
Note that
\[
E[\xi^{(\cdot)}] - \sum_{k=1}^{K-1} k P[\xi^{(\cdot)} = k] = \sum_{k=K}^{\infty} k P[\xi^{(\cdot)} = k],
\]
and as $B$ upper bounds $\Vert f(\cdot) \Vert_1$, then the bias $\EE[\hat{F}^{(K_\text{dyn},1)}(\bW) - \hat{F}^{(K,R)}(\bW) ]$ can be at most $(E[\xi^{(\cdot)}] - \sum_{k=1}^{K-1} k P[\xi^{(\cdot)} = k]) \cdot B$.

\end{proof}

\subsection*{Theorem~\ref{t:geotail}}
\begin{proof}[Proof of Theorem~\ref{t:geotail}]
Note that the condition 
$$
\inf_{\bx \in \Omega \backslash \cS} \sum_{\by \in \cS} p_\Phi (\bx,\by) \geq \epsilon
$$
ensures that the MC $\Phi'(\bW)$ satisfies Doeblin's condition,
and therefore $\Phi'(\bW)$ is geometrically ergodic with convergence rate $(1-\epsilon)$~\cite[pp.\ 30]{Strook1995}.
Finally, by Kendall's theorem~\cite[Theorem 15.1.1]{meyn2012markov}, a geometric ergodicity and a geometric decay in the tail of the return time distribution are equivalent conditions.
\end{proof}

\subsection*{Proof of Corollary~\ref{c:lv}}
\begin{proof}
An unbiased estimate of $\nabla_\bW \cL_Z$ for one tour is obtained from $F^{(K_\text{dyn},R)}(\bW,f)$ in eq.~\eqref{e:FprimeSuper} of Theorem~\ref{t:unbiasSuperF} with $f(\by) =  \frac{1}{N} \sum_{n=1}^N \frac{\partial E(\bx_n; \bW)}{\partial \bW} - \frac{\partial E(\by; \bW) }{\partial \bW} $. Averaging the gradient of each tour over $R \geq 1$ tours gives the desired result.
\end{proof}
\subsection*{Source Code}
Our source code and detailed results are hosted at \texttt{https://github.com/PurdueMINDS/MCLV-RBM}.

\balance
\fontsize{9pt}{10pt} \selectfont
\bibliography{Saravese_Ribeiro2017}
\bibliographystyle{aaai}
\end{document}